\documentclass[11pt,reqno]{article}
\usepackage{fullpage,graphicx,xcolor}
\usepackage[colorlinks,citecolor=blue,urlcolor=blue,breaklinks]{hyperref}
\usepackage{amssymb}
\usepackage{amsthm} % before newtxmath because of \openbox
\usepackage{amsmath} %@thomas to allow tags in eqnarray

\usepackage{mathtools}
\mathtoolsset{showonlyrefs}
\usepackage{xparse}

\usepackage{authblk}

\usepackage[utf8]{inputenc}
\usepackage[T1]{fontenc}
\usepackage{newtxtext}
\usepackage{booktabs}       % professional-quality tables
\usepackage{microtype}      % microtypography
\usepackage[citestyle=alphabetic,bibstyle=alphabetic]{biblatex}
\usepackage{etoolbox}
\usepackage[enable]{easy-todo}

\usepackage{graphicx} 
\usepackage{float}
\usepackage{subcaption}
\usepackage{multirow}

\usepackage[linesnumbered,ruled,vlined]{algorithm2e}

\let \oldsection \section
\renewcommand{\section}{\vspace{3ex plus 1ex}\oldsection}

\usepackage{empheq}
\usepackage{enumitem}
\usepackage{bbm}
\usepackage{relsize}

\definecolor{ddarkbrown}{rgb}{0.5,0.2,0.05} \definecolor{bbluegray}{rgb}{0.05,0,0.5}

\newtheorem{theorem}{Theorem}[section]

\newtheorem{definition}[theorem]{Definition}
\newtheorem{lemma}[theorem]{Lemma}
\newtheorem{remark}[theorem]{Remark}
\newtheorem{corollary}[theorem]{Corollary}
\newcommand{\BEAS}{\begin{eqnarray*}}
\newcommand{\EEAS}{\end{eqnarray*}}
\newcommand{\BEA}{\begin{eqnarray}}
\newcommand{\EEA}{\end{eqnarray}}
\newcommand{\BEQ}{\begin{equation}}
\newcommand{\EEQ}{\end{equation}}
\newcommand{\BIT}{\begin{itemize}}
\newcommand{\EIT}{\end{itemize}}
\newcommand{\BNUM}{\begin{enumerate}}
\newcommand{\ENUM}{\end{enumerate}}

\newcommand{\BA}{\begin{array}}
\newcommand{\EA}{\end{array}}

\newcommand{\argmax}{\mathop{\rm argmax}}

\title{Linear Bandits on Uniformly Convex Sets}

\author[1,2]{Thomas Kerdreux\footnote{Corresponding author, thomaskerdreux at gmail.com.}}
\author[2]{Christophe Roux}
\author[3,4]{Alexandre d'Aspremont}
\author[1,2]{Sebastian Pokutta}
\affil[1]{Zuse Institute, Berlin, Germany}
\affil[2]{Technische Universit{\"a}t, Berlin, Germany}
\affil[3]{CNRS, UMR 8548}
\affil[4]{D.I., \'Ecole Normale Sup\'erieure, Paris, France}

\date{\today}
\begin{document}

\maketitle

\begin{abstract}
Linear bandit algorithms yield $\tilde{\mathcal{O}}(n\sqrt{T})$ pseudo-regret bounds on compact convex action sets $\mathcal{K}\subset\mathbb{R}^n$ and two types of structural assumptions lead to better pseudo-regret bounds. When $\mathcal{K}$ is the simplex or an $\ell_p$ ball with $p\in]1,2]$, there exist bandits algorithms with $\tilde{\mathcal{O}}(\sqrt{nT})$ pseudo-regret bounds. Here, we derive bandit algorithms for some strongly convex sets beyond $\ell_p$ balls that enjoy pseudo-regret bounds of $\tilde{\mathcal{O}}(\sqrt{nT})$, which answers an open question from \parencite[\S 5.5.]{bubeck2012regret}. Interestingly, when the action set is uniformly convex but not necessarily strongly convex, we obtain pseudo-regret bounds with a dimension dependency smaller than $\mathcal{O}(\sqrt{n})$. However, this comes at the expense of asymptotic rates in $T$ varying between $\tilde{\mathcal{O}}(\sqrt{T})$ and $\tilde{\mathcal{O}}(T)$.
\end{abstract}

\section{Introduction}
We consider online linear learning with partial information, a.k.a.~the \textit{linear bandit problem}.
At each round $t\leq T$, the \emph{player} (the bandit algorithm) chooses $a_t\in\mathcal{K}\subset\mathbb{R}^n$ and an \emph{adversary} simultaneously decides on a loss vector $c_t\in\mathbb{R}^n$ (loss is linear).
The player then observes its loss $\langle c_t; a_t\rangle$ but does not have access $c_t$. 
The goal of the player is to minimize its cumulative loss $\sum_{t=1}^{T}\langle c_t;a_t\rangle$. 
The \emph{regret} $R_t$ compares this cumulative loss against the cumulative loss of the best single action in hindsight, \textit{i.e.},
\begin{equation}\label{eq:regret}\tag{Regret}
R_T(\mathcal{K}) \triangleq \sum_{t=1}^{T}\langle c_t ; a_t\rangle - \underset{a\in\mathcal{K}}{\text{min }} \sum_{t=1}^{T}\langle c_t ; a\rangle.
\end{equation}
Bandit algorithms use internal randomization to obtain sub-linear regret upper bounds.
There exist several notions of regret to monitor the performance of bandit algorithms. 
The expected regret or an upper bound on \eqref{eq:regret} with high-probability are the most meaningful, yet challenging to obtain.
Hence, the weaker notion of \emph{pseudo-regret} is often considered as a good proxy for measuring the bandit performance \parencite{bubeck2012regret}. 
It serves as a motivation to design new bandit algorithms.
Let us write $\mathbb{E}$ the expectation w.r.t.~the randomness of the bandit action only, we have
\begin{equation}\label{eq:pseudo_regret}\tag{Pseudo-Regret}
\Bar{R}_T(\mathcal{K}) \triangleq \mathbb{E}\sum_{i=1}^{T} \langle c_t ; a_t\rangle - \underset{a\in\mathcal{K}}{\text{min }}\mathbb{E}\sum_{i=1}^{T} \langle c_t ; a\rangle.
\end{equation}
% Strategie of the bandit is with randomness etc.. 
We make the \textit{bounded scalar loss assumption}, \textit{i.e.}, $c_t$ is such that $\langle c_t; a\rangle \leq 1$ for any $a\in\mathcal{K}$.
In particular, it means that $c_t$ belongs to the \emph{polar} $\mathcal{K}^\circ\triangleq \{d\in\mathbb{R}^n~|~\langle d;x\rangle \leq 1, \forall x\in\mathcal{K}\}$ of $\mathcal{K}$.

There exist bandit algorithms with $\tilde{O}(n\sqrt{T})$ upper bounds on the \emph{pseudo-regret} for general compact convex sets $\mathcal{K}$ \parencite{bubeck2012regret}.
However, since the loss is linear, it is not possible to leverage the lower curvature (\textit{e.g.}, the strong convexity) of the loss function to obtain improved pseudo-regret bounds.
Instead, the bandit algorithm can only leverage the specific structure of the action set $\mathcal{K}$.
To the best of our knowledge, only two structures are known to induce faster pseudo-regret bounds of $\tilde{O}(\sqrt{nT})$:
when $\mathcal{K}$ is a simplex or an $\ell_{p}$ ball with $p\in]1,2]$ \parencite{bubeck2018sparsity}.
In each of these cases, the analysis relies on explicit analytical formulas of the action set rather than on generic quantitative properties, \textit{e.g.}, the strong convexity of the set.

Our goal here is to design bandit algorithms that achieve pseudo-regret of $\tilde{\mathcal{O}}(\sqrt{nT})$ (resp. $\tilde{\mathcal{O}}(n^{1/q}T^{1/p})$) when the set $\mathcal{K}$ is strongly convex (resp. $q$-uniformly convex with $q\geq 2$ and $p$ s.t.~$1/p + 1/q=1$).
The uniform convexity of a set is a measure of the set upper curvature that subsumes strong convexity.
For instance, the $\ell_p$ balls are strongly convex (Definition \ref{def:strong_convexity_set}) for $p\in]1,2]$ but only uniformly convex (Definition \ref{def:uniform_convexity_set}) for $p\geq 2$.

\paragraph{Related Work.}
Linear bandit algorithms are applied in a variety of applications. We detail one of them, which was our initial research motivation.
Linear Bandit algorithms are instrumental in solving minimax problems with convex-linear structure stemming from learning applications, see, \textit{e.g.}, SVMs \parencite{hazan2011beating,clarkson2012sublinear} or Distributional Robust Optimization \parencite{namkoong2016stochastic,curi2020adaptive}.
In these settings, the minimax variable's linear part is a probability distribution over the dataset of size $n$.
The linear bandit algorithms provide a principled framework to adaptively sample a fraction of the dataset per iteration while ensuring the convergence to a minimax optimum. 
The iterations' cost of the minimax algorithm is then favorably dependent on the size $n$ of the dataset.
However, the dimension dependency of the linear bandit algorithm's regret bound now appears in the minimax method's convergence rate, making it crucial to design linear bandit algorithms with favorable dimension-dependent regret bounds.

Significant focus has been dedicated to designing efficient algorithms (in the full and partial feedback setting) leveraging additional properties of the loss functions such as smoothness, strong convexity \parencite{saha2011improved,hazan2014bandit,garber2020improved,Garber20} with much less attention to the corresponding structural assumptions on the action sets.
By studying the effect of uniform convexity of the action set in the bandit setting, we contribute to filling this gap.
Note that some works recently relied on smoothness \parencite{levy2019projection} or uniform convexity assumptions on the set in online linear learning \parencite{huang2016following,huang2017following,Molinaro2020,kerdreux2020projection} or ``online learning with a hint'' \parencite{dekel2017online,bhaskara2020online,bhaskara2020onlineMany}.

At a high level, our work shares some similarities with \parencite{d2018optimal} for affine-invariant analysis of accelerated first-order methods or with \parencite{srebro2011universality,rakhlin2017equivalence} in the full-information setting.
Indeed, they link regret bounds of online mirror descent algorithms with the Martingale type of the ambient space.
Here, we instead rely on the uniform convexity of the action set. 
It is a more intuitive yet stronger requirement, for an explanation see, \textit{e.g.}, \parencite{kerdreux21}.

\paragraph{Contribution.}
Our contribution are three-fold.
\begin{enumerate}
    \item We propose a barrier function $F_{\mathcal{K}}$ for the bandit problem with strongly convex sets (more generally uniformly convex sets), \textit{i.e.}, for $x\in\text{int}(\mathcal{K})$
    \begin{equation}\label{eq:barrier_function}\tag{Barrier}
    F_{\mathcal{K}}(x) \triangleq -\ln(1-\|x\|_{\mathcal{K}}) - \|x\|_{\mathcal{K}},
    \end{equation}
    where $\|\cdot\|_{\mathcal{K}}$ is the \emph{gauge function} to $\mathcal{K}$. 
    For $x\in\mathbb{R}^n$, it is defined as
    \begin{equation}\label{eq:gauge_function}\tag{Gauge}
    \|x\|_{\mathcal{K}} = \text{inf}\big\{ \lambda >0 ~|~  x \in\lambda\mathcal{K} \big\}.
    \end{equation}
    
    \item In Theorem \ref{th:linear_bandit_strongly_convex} we provide a pseudo-regret upper bound $\tilde{\mathcal{O}}(\sqrt{nT})$ for a linear bandit algorithm on some strongly convex sets.
    To the best of our knowledge, this setting has never been studied except in the specific case of the $\ell_{p}$ balls with $p\in]1,2]$. 
    Importantly, this drastically extends the family of actions sets, \textit{i.e.}, besides the simplex and the $\ell_p$ balls with $p\in]1,2]$, with such improved dimension dependency of the pseudo-regret bound in $\mathcal{O}(\sqrt{n})$.
    This is an answer to the open question from \parencite[\S 5.5.]{bubeck2012regret}.
    
    \item When the action set is $(\alpha,q)$-uniformly convex with $q\geq2$, we prove in Theorem \ref{th:linear_bandit_UC} a pseudo-regret bound of $\tilde{\mathcal{O}}(n^{1/q}T^{1/p})$ with $p\in]1,2]$ s.t.~$1/p + 1/q =1$.
    This trade-off means that it is possible to obtain a pseudo-regret bound with dimension dependence faster than $\mathcal{O}(\sqrt{n})$ balanced by a slower (w.r.t.~$\tilde{\mathcal{O}}(\sqrt{T})$) asymptotical regime in $T$.
    While counter-intuitive at first, this can be of interest, \textit{e.g.}, in minimax problems where the dataset is large.
\end{enumerate}

\paragraph{Outline.}
In Section \ref{sec:preliminaries} we introduce the structural assumptions on the action sets $\mathcal{K}$ and provide some elementary results linking these structures with important quantities in the analysis.
In Section \ref{ssec:algorithm} we describe the classical Mirror Descent type algorithm for bandits and our design of barrier function for uniformly convex sets.
In Section \ref{ssec:main_results}, we then provide the convergence rates when the action sets are strongly convex (Theorem \ref{th:linear_bandit_strongly_convex}) and uniformly convex (Theorem \ref{th:linear_bandit_UC}).
In Section \ref{ssec:technical_lemma}, we present the main technical lemmas. 
Finally, in Appendix \ref{app:uc_equivalence_simplified} we prove the main link between uniform convexity of the set $\mathcal{K}$ and upper bounds of the Bregman Divergence of a specific function.

\paragraph{Notations.}
Let $\mathbb{R}^n$ be the ambient space and $(e_i)$ its canonical basis.
For a norm $\|\cdot\|$, we write $\|d\|^\star=\text{sup}_{\|x\|\leq 1} \langle x; d \rangle$ for its \emph{dual norm}.
For a convex function $f$, we write $f^*(d)\triangleq \text{sup}_{x\in\mathbb{R}^n} \langle x; d\rangle - f(x)$ its \emph{Fenchel conjugate}.
Let 
$\ell_{\infty}(R)$ be the infinity ball with radius $R>0$ and $\ell_1(r)$ the norm ball of $\|x\|_1=\sum_{i=1}^{n}{|x_i|}$ with radius $r>0$.
For an open set $\mathcal{D}\subset\mathbb{R}^n$, we write $\Bar{\mathcal{D}}$ its \emph{closure}.
For a compact convex set $\mathcal{K}$, we write $\partial\mathcal{K}$ its \emph{boundary} and $\text{Int}(\mathcal{K})$ its \emph{interior}.
$N_{\mathcal{K}}(x)\triangleq \big\{d\in\mathbb{R}^n~|~\langle x-y ; d\rangle \geq 0~~\forall y\in\mathcal{K}\big\}$ is the \emph{normal cone} of $\mathcal{K}$ at $x$. 
We consider fully-dimensional compact convex sets $\mathcal{K}\subset\mathbb{R}^n$ s.t.~$\ell_1(r)\subset\mathcal{K}\subset\ell_{\infty}(R)$ for some $r,R>0$ which are a priori numerical constant, in particular not depending on the dimension $n$.
For $d\in\mathbb{R}^n$, we write $\sigma_{\mathcal{K}}(d)\triangleq \text{sup}_{x\in\mathcal{K}}\langle x; d\rangle$  the \emph{support function} $\sigma_{\mathcal{K}}$ of $\mathcal{K}$.
We have $\|\cdot\|_{\mathcal{K}}^\star=\sigma_{\mathcal{K}}$.
Recall that we write $\mathcal{K}^\circ= \{d\in\mathbb{R}^n~|~\langle d;x\rangle \leq 1, \forall x\in\mathcal{K}\}$ the \emph{polar} of $\mathcal{K}$.
Note that we have $\|\cdot\|_{\mathcal{K}^\circ}=\sigma_{\mathcal{K}}$.
We write $X\sim\text{Ber}(p)$ (resp. $\text{Rademacher}(p)$) a random variable $X$ following a Bernoulli (a Rademacher), \textit{i.e.}, with values in $\{0,1\}$ (resp. $\{-1,1\}$) and $\mathbb{P}(X=1) = p$.

\section{Preliminaries}\label{sec:preliminaries}
In this section, we introduce the structural assumption on $\mathcal{K}$ we will consider. 
Note that we will assume set smoothness (Definition \ref{def:smooth_set}) simply to ensure that \eqref{eq:barrier_function} is differentiable. 
On the contrary, the strong convexity (Definition \ref{def:strong_convexity_set}) is the structure that allows for the $\sqrt{n}$ acceleration in the pseudo-regret bounds.
Then we review the link between the structure of $\mathcal{K}$ and the differentiability of the set gauge function \eqref{eq:gauge_function} which then allows us to study the properties of the proposed barrier.
Finally, we link upper bounds on some Bregman distance with the strong convexity of some set in Lemma \ref{lem:bregman_upper_bound}.
This will be a key inequality in our analysis.\\

\noindent A convex differentiable function $f$ is $L$-smooth on $\mathcal{K}$ w.r.t.~$\|\cdot\|$ if and only if for any $(x,y)\in\mathcal{K}\times\mathcal{K}$
\begin{equation}\label{eq:L_smoothness_f}\tag{Smoothness}
f(y) \leq f(x) + \langle \nabla f(x); y-x \rangle + \frac{L}{2} \|y-x\|^2.
\end{equation}
The H\"older smoothness of a function is a relaxation of \eqref{eq:L_smoothness_f}. 
For $p\in]1,2]$, a convex differentiable function $f$ is $(L, p)$-H\"older smooth w.r.t.~$\|\cdot\|$ if and only if for any $(x,y)\in\mathcal{K}\times\mathcal{K}$
\begin{equation}\label{eq:L_Holder_smoothness_f}\tag{H\"older-Smoothness}
f(y) \leq f(x) + \langle \nabla f(x); y-x \rangle + \frac{L}{p} \|y-x\|^p. 
\end{equation}
On the other hand, a set $\mathcal{K}$ is \emph{smooth} when there is exactly one supporting hyperplane at each point of its boundary $\partial\mathcal{K}$ \parencite{schneider2014convex}.
This can be defined as follows.
\begin{definition}[Smooth Set]\label{def:smooth_set}
A compact convex set $\mathcal{K}$ is smooth if and only if $|N_{\mathcal{K}}(x)\cap\partial\mathcal{K}^\circ|=1$ for any $x\in\partial\mathcal{K}$.
\end{definition}
One should be cautious not to confuse the smoothness of $f$ as defined in \eqref{eq:L_smoothness_f} and the smoothness of $\mathcal{K}$ as defined in Definition \ref{def:smooth_set}.
Indeed, the smoothness of the set is a much weaker notion as, for instance, it implies only the differentiability of $\sigma_{\mathcal{K}}(\cdot)$, see Lemma \ref{eq:differentiability_gauge}.
Note that not all strongly convex set are smooth.
For instance, the $\ell_p$ or the $p$-Schatten balls for $p\in]1,2]$ are smooth and strongly convex but the $\ell_{1,2}$ ball (Elastic-Net constraints) is strongly convex but not smooth.
Also, the smoothness and strict convexity of a set are dual properties to each other in the following sense \parencite[\S 26]{kothe1983topological}
\begin{lemma}[Duality Set Smoothness and Strict Sonvexity]\label{lem:duality_smoothness_convexity}
Consider a compact convex set $\mathcal{K}\subset\mathbb{R}^n$.
Then, $\mathcal{K}$ is strictly convex if and only if $\mathcal{K}^\circ$ is smooth.
\end{lemma}
\begin{proof}
Let us recall the proof for completeness. 
Assume $\mathcal{K}$ is strictly convex and let $d\in\partial\mathcal{K}^\circ$. Let $x_1,x_2\in\partial\mathcal{K}\cap N_{\mathcal{K}^\circ}(d)$.
By definition of the normal cone, we have $\langle d;x_i \rangle \geq \langle d^\prime;x_i \rangle $ for any $d^\prime\in\mathcal{K}^\circ$ and $i=1,2$.
Hence, $\langle d;x_i \rangle = \text{sup}_{d^\prime\in\mathcal{K}^\circ} \langle d^\prime;x_i \rangle = \|x_i\|_{\mathcal{K}}=1$ so that 
\[
1 = \langle d; (x_1 + x_2)/2 \rangle \leq \|d\|_{\mathcal{K}^\circ} \|(x_1 + x_2)/2\|_{\mathcal{K}} = \|(x_1 + x_2)/2\|_{\mathcal{K}},
\]
and we conclude that $(x_1+x_2)/2\in\partial\mathcal{K}$ and by strict convexity of $\mathcal{K}$, $x_1=x_2$ which concludes.
Alternatively, assume that $\mathcal{K}^\circ$ is smooth. 
Assume by the absurd that there exists distinct $x_1,x_2\in\partial\mathcal{K}$ s.t. $(x_1+x_2)/2\in\partial\mathcal{K}$ and let $d\in N_{\mathcal{K}}((x_1+x_2)/2)\cap\partial\mathcal{K}^\circ$. 
Then, by convexity $d\in N_{\mathcal{K}}(x_i)$ for $i=1,2$ and $\langle d; x_i\rangle = 1$. 
In particular, this means that $x_i\in N_{\mathcal{K}^\circ}(d)\cap\partial\mathcal{K}$ for $i=1,2$ and contradicts the smoothness of $\mathcal{K}^\circ$.
\end{proof}

\begin{definition}[Set Strong Convexity]\label{def:strong_convexity_set}
Let $\mathcal{K}$ be a centrally symmetric set with non-empty interior and $\alpha>0$. $\mathcal{K}$ is $\alpha$-uniformly convex w.r.t.~$\|\cdot\|_{\mathcal{K}}$ if and only if for any $x,y,z\in\mathcal{K}$ and $\gamma\in[0,1]$ we have
\begin{equation}\label{eq:strong_convexity_set}\tag{Set Strong Convexity}
\big(\gamma x + (1-\gamma)y + \frac{\alpha}{2} \gamma (1-\gamma) \|x-y\|^2_{\mathcal{K}}z\big)\in\mathcal{K}.
\end{equation}
\end{definition}

More generally, we can define the uniform convexity of a set $\mathcal{K}$ which subsumes the strong convexity.
For instance the $\ell_p$ balls with $p>2$ are uniformly convex but not strongly convex.

\begin{definition}[Set Uniform Convexity]\label{def:uniform_convexity_set}
Let $\mathcal{K}$ be a centrally symmetric set with non-empty interior, $\alpha>0$, and $q\geq 2$. $\mathcal{K}$ is $(\alpha, q)$-uniformly convex w.r.t.~$\|\cdot\|_{\mathcal{K}}$ if and only if for any $x,y,z\in\mathcal{K}$ and $\gamma\in[0,1]$ we have
\begin{equation}\label{eq:def_uniform_convexity}\tag{Set Uniform Convexity}
\big(\gamma x + (1-\gamma)y + \frac{\alpha}{q} \gamma (1-\gamma) \|x-y\|^q_{\mathcal{K}}z\big)\in\mathcal{K}.
\end{equation}
\end{definition}

We now recall the geometrical condition on $\mathcal{K}$ that is equivalent to differentiability of $\mathcal{K}$ \parencite[Corollary 1.7.3.]{schneider2014convex}.

\begin{lemma}[Gauge Differentiability]\label{eq:differentiability_gauge}
A gauge function $\|\cdot\|_{\mathcal{K}}$ \eqref{eq:gauge_function} is differentiable at $x\in\mathbb{R}^n\setminus\{0\}$ if and only if its support set
\begin{equation}\label{eq:support_set}\tag{Support Set}
S(\mathcal{K}^\circ,x)\triangleq \{d\in\mathcal{K}^\circ~:~\langle d; x\rangle = \underset{d^\prime\in\mathcal{K}^\circ}{\sup }\langle d^\prime; x\rangle \},
\end{equation}
contains a single point $d$.
If this is the case, we have $\nabla \|\cdot\|_{\mathcal{K}}(x)=d$.
Besides, the following assertions are true
\begin{enumerate}[label=(\alph*)]
    \item $\big\|\big(\nabla\|\cdot\|_{\mathcal{K}}(x)\big)\big\|_{\mathcal{K}^\circ} = 1$, \textit{i.e.}, $\nabla\|\cdot\|_{\mathcal{K}}(x)\in\mathcal{K}^\circ$. \label{item:unit_gradient}
    
    \item For $\lambda>0$, $\nabla \|\cdot\|_{\mathcal{K}}(\lambda x)=\nabla \|\cdot\|_{\mathcal{K}}(x)$.\label{item:unit_gradient_homogeneity}
    
    \item If $\mathcal{K}^\circ$ is strictly convex then $\|\cdot\|_{\mathcal{K}}$ is differentiable on $\mathbb{R}^n\setminus\{0\}$.\label{item:set_smooth_differentiable}
\end{enumerate}

\end{lemma}
\begin{proof}
The differentiability result for $\|\cdot\|_{\mathcal{K}}$ comes from \parencite[Corollary 1.7.3.]{schneider2014convex}, where we used that $\|\cdot\|_{\mathcal{K}} = \sigma_{\mathcal{K}^\circ}$.
\ref{item:unit_gradient} follows from the fact that the supremum in \eqref{eq:support_set} is attained at $\partial\mathcal{K}^\circ$. 
For $\lambda >0$, we have $S(K^\circ,\lambda x)=S(K^\circ,x)$ and hence \ref{item:unit_gradient_homogeneity}.
Now assume that $\mathcal{K}^\circ$ is strictly convex and consider $x\in\mathbb{R}^n\setminus\{0\}$.
First remark as for \ref{item:unit_gradient} that $S(\mathcal{K}^\circ,x)\subset\partial\mathcal{K}^\circ$. 
Assume that $|S(\mathcal{K}^\circ,x)|\neq1$. 
Then, for $d_1,d_2$ distinct in $S(\mathcal{K}^\circ,x)$, we have $[d_1,d_2]\subset S(\mathcal{K}^\circ,x)\subset \partial\mathcal{K}^\circ$ which then contradicts the strict convexity of $\mathcal{K}^\circ$. Hence $|S(\mathcal{K}^\circ,x)|=1$ which concludes \ref{item:set_smooth_differentiable}.
\end{proof}

\begin{definition}[Bregman Divergence]
The Bregman divergence of $F:\mathcal{D}\rightarrow\mathbb{R}$ is defined for $(x,y)\in\Bar{\mathcal{D}}\times\mathcal{D}$ by
\begin{equation}\label{eq:bregman_divergence_definition}\tag{Bregman Divergence}
D_F(x,y) = F(x) - F(y) - \langle x-y ; \nabla F(y)\rangle.
\end{equation}
\end{definition}
The strong-convexity assumption on $\mathcal{K}$ appears in the analysis of Algorithm \ref{algo:bandit_mirror_descent} via an upper bound on the \eqref{eq:bregman_divergence_definition} of $\frac{1}{2}\|\cdot\|_{\mathcal{K}^\circ}$.
Indeed, when $\mathcal{K}$ is \emph{strongly convex}, then $\mathcal{K}^\circ$ is strongly smooth and hence $\sigma_{\mathcal{K}}^2$ is $L$-smooth with respect to $\|\cdot\|_{\mathcal{K}}$, see \parencite[Theorem 4.1.]{kerdreux21} that we recall in Theorem \ref{th:UC_support} in the Appendix \ref{app:uc_equivalence_simplified}.
It then implies the following quadratic upper bound on its Bregman Divergence.

\begin{lemma}[Upper-bound on the Bregman Divergence of $\frac{1}{2}\|\cdot\|^2_{\mathcal{K}^\circ}$]\label{lem:bregman_upper_bound}
Let $q\geq 2$ and $p\in]1,2]$ s.t.~$1/p +  1/q =1$.
Let $\mathcal{K}$ be a centrally symmetric set with non-empty interior. 
Assume $\mathcal{K}$ is $(\alpha, q)$-uniformly convex with respect to $\|\cdot\|_{\mathcal{K}}$. Then, for any $(u,v)\in\mathbb{R}^n$, we have
\begin{equation}\label{eq:bregman_upper_bound}
D_{\frac{1}{2}\|\cdot\|^2_{\mathcal{K}^\circ}}(u,v) \leq 2p\big(1 + (q/(2\alpha))^{1/(q-1)}\big) \|u-v\|^{p}_{\mathcal{K}^\circ}.
\end{equation}
\end{lemma}
\begin{proof}
For a $(L,r)$-H\"older smooth function $f$ w.r.t.~to $\|\cdot\|$ we immediately have $D_{f}(u,v)\leq \frac{L}{r}\|u-v\|^r$.
Theorem \ref{th:UC_support} implies that $\frac{1}{2}\|\cdot\|^2_{\mathcal{K}^\circ}$ is 
$(L, p)$-H\"older Smooth on $\mathcal{K}^\circ$ w.r.t.~$\|\cdot\|_{\mathcal{K}^\circ}$ where $L = 2p\big(1 + \big(\frac{q}{2\alpha}\big)^{1/(q-1)}\big).$
This concludes the proof.
\end{proof}
We immediately obtain the following corollary for the strongly convex case with $p=q=2$.

\begin{corollary}[Strongly Convex Case]\label{cor:bregman_upper_bound_strong_convexity}
Let $\mathcal{K}$ be a centrally symmetric set with non-empty interior. 
Assume $\mathcal{K}$ is $\alpha$-strongly convex with respect to $\|\cdot\|_{\mathcal{K}}$. 
Then for any $(u,v)\in\mathbb{R}^n$, we have
\begin{equation}\label{eq:bregman_upper_bound_strong_convex_case}
D_{\frac{1}{2}\|\cdot\|^2_{\mathcal{K}^\circ}}(u,v) \leq 4 \Big(\frac{\alpha+1}{\alpha}\Big) \big\|u-v\big\|^{2}_{\mathcal{K}^\circ}.
\end{equation}

\end{corollary}

\section{Pseudo-Regret Bounds of Linear Bandit on Strongly Convex Sets}\label{sec:pseudo_regret_linear_bandit}

In Section \ref{ssec:algorithm}, we first present the algorithm and barrier function for linear bandits on uniformly convex sets.
In Section \ref{ssec:main_results}, we then present the main pseudo-regret bounds and the proofs of the technical lemmas are relegated in Section \ref{ssec:technical_lemma}.

\subsection{Mirror Descent for Bandits}\label{ssec:algorithm}
We propose to use a similar bandit algorithm to the one developed in \parencite{bubeck2012regret} for linear bandits over the Euclidean ball.
% As \parencite{bubeck2012regret} for linear bandits with the Euclidean ball, 
Namely, Algorithm \ref{algo:bandit_mirror_descent} is an instantiation of Online Stochastic Mirror Descent (OSMD) with a carefully designed barrier function $F_{\mathcal{K}}: \text{Int}(\mathcal{K})\rightarrow \mathbb{R}^+$. 
For any $x\in\text{Int}(\mathcal{K})$ we defined in \eqref{eq:barrier_function}
\[
    F_{\mathcal{K}}(x) = -\ln(1-\|x\|_{\mathcal{K}}) - \|x\|_{\mathcal{K}}.
\]
Algorithm \ref{algo:bandit_mirror_descent} keeps track of a sequence of vectors $x_t\in(1-\gamma)\mathcal{K}$ and at each iteration samples an action $a_t\in\mathcal{K}$ as described in Lines \ref{line:sampling_random_variables}-\ref{line:bandit_action_bis}.
For some $r>0$, we assume $\ell_1(r)\subset\mathcal{K}$ so that $re_i\in\mathcal{K}$.
After playing action $a_t\in\mathcal{K}$, the bandit receives the loss $\langle c_t; a_t\rangle$ associated to its action without observing the full vector $c_t\in\mathcal{K}^\circ$.
In Line \ref{line:estimation_loss}, it then proposes an unbiased estimation $\tilde{c}_t$ of $c_t$.
Indeed, we have (because $\mathbb{P}(\xi_t=0)= 1 - \|x\|_{\mathcal{K}}$)
\[
\mathbb{E}_{\xi_t,i_t,\epsilon_t}(\tilde{c_t}) = \mathbb{P}(\xi_t=0) \sum_{i=1}^{n}{\frac{n}{r^2}\frac{1}{n} \Big[ \frac{\langle r e_i ; c_t\rangle }{2(1-\|x\|_{\mathcal{K}})}r e_i + \frac{\langle -r e_i;c_t\rangle}{2(1-\|x\|_{\mathcal{K}})}(-r e_i) \Big]}=c_t.
\]
The bandit then provides the vector $\tilde{c}_t$ to an online learning algorithm that updates the $x_t$ vector in Line \ref{line:mirror_descent_step}. 
Importantly, because $x_t\in(1-\gamma)\mathcal{K}$ with $\gamma\in]1,2[$ we have $\|x_t\|<1$ so that $\nabla F_{\mathcal{K}}(x_t)$ is well defined.\\

\begin{algorithm}[H]
\SetAlgoLined
\KwIn{$\eta>0$, $\gamma\in]0,1[$, $\mathcal{K}$ smooth and strictly convex s.t.~$\ell_1(r)\subset\mathcal{K}$.}
\textbf{Barrier:} $F_{\mathcal{K}}(\cdot) = -\ln(1 - \|\cdot\|_{\mathcal{K}}) - \|\cdot\|_{\mathcal{K}}.$\label{line:barrier}\\
\textbf{Initialize:} $x_1\in\text{argmin}_{x\in(1-\gamma)\mathcal{K}}F_{\mathcal{K}}(x)$.\label{line:initialization}\\
\For{$t \gets 1, \ldots, T $}{
 Sample $\xi_t\sim \text{Ber}(\|x_t\|_{\mathcal{K}})$, $i_t\sim \text{Uniform}(n)$ and $\epsilon_t\sim\text{Rademacher}(\frac{1}{2})$.
 \label{line:sampling_random_variables} 
 
 \eIf{$\xi_t=1$}{
   $a_t \gets x_t/\|x_t\|_{\mathcal{K}}.\hfill \vartriangleright\text{Define bandit action.}$\label{line:bandit_action}
   }{
   $a_t\gets r \epsilon_t e_{i_t}$.\label{line:bandit_action_bis}
  }
  
 $\mathlarger{\tilde{c}_t \gets \frac{n}{r^2} (1-\xi_t) \frac{\langle a_t;c_t\rangle }{1-\|x_t\|_{\mathcal{K}}}a_t.} \hfill \vartriangleright\text{Estimate full loss vector }c_t.$\label{line:estimation_loss}

 $\mathlarger{x_{t+1} \gets \underset{y\in(1-\gamma)\mathcal{K}}{\text{argmin }} D_{F_{\mathcal{K}}}\big(y, \nabla F_{\mathcal{K}}^*(\nabla F_{\mathcal{K}}(x_t)- \eta \tilde{c}_t)\big) \big)}.\hfill \vartriangleright\text{Mirror Descent step. }$\label{line:mirror_descent_step}
 }
 \KwOut{$\frac{1}{T}\sum_{t=1}^{T}a_t$}
  \caption{Bandit Mirror Descent (BMD) on some Curved Sets $\mathcal{K}\subset\mathbb{R}^n$}\label{algo:bandit_mirror_descent}
\end{algorithm}

\bigskip
To ensure that Line \ref{line:mirror_descent_step} of Algorithm \ref{algo:bandit_mirror_descent} is well defined, we need to check, \textit{e.g.}, that all $x_t$ belongs of $\text{Int}(\mathcal{K})$ (which we know is the case because $x_t\in(1-\gamma)\mathcal{K}$) or that $\nabla F_{\mathcal{K}}(x_t)-\eta\tilde{c}_T$ belongs to $\mathcal{D}_{\mathcal{K}}^*$ the domain where $F^*_{\mathcal{K}}$ is defined.
In Lemma \ref{lem:check_definition_barrier} below, we guarantee that Algorithm \ref{algo:bandit_mirror_descent} is well defined.
We also prove that $F_{\mathcal{K}}$ is Legendre (Definition \ref{def:legendre_function}) which allows us   to invoke classical convergence results as in \parencite{bubeck2012regret}.

\begin{definition}[Legendre Function]\label{def:legendre_function}
A continuous function $F:\bar{\mathcal{D}}\rightarrow \mathbb{R}$ is \emph{Legendre} if and only if
\begin{enumerate}[label=(\alph*)]
    \item $F$ is strictly convex and admits continuous first partial derivatives on $\mathcal{D}$.
    
    \item $\underset{x\rightarrow\Bar{\mathcal{D}}\setminus\mathcal{D}}{\lim } \|\nabla F(x)\| = + \infty.$
\end{enumerate}
\end{definition}

\begin{lemma}[Barrier $F_{\mathcal{K}}$ for $\mathcal{K}$]\label{lem:check_definition_barrier}
Consider a compact, smooth and strictly convex $\mathcal{K}$.
We consider for $x\in \mathcal{D}_{\mathcal{K}}\triangleq \big\{x\in\mathbb{R}^n~|~\|x\|_{\mathcal{K}}<1\big\}$ the following barrier function as defined in \eqref{eq:barrier_function}
\[
F_{\mathcal{K}}(x) = -\ln(1 - \|x\|_{\mathcal{K}}) - \|x\|_{\mathcal{K}}.
\]
Then $F$ is Legendre (Definition \ref{def:legendre_function}) with $\mathcal{D}_{\mathcal{K}}^* = \mathbb{R}^n$ and $\mathcal{K}\subset\bar{\mathcal{D}}_{\mathcal{K}}$.
\end{lemma}
\begin{proof}
From Lemma \ref{eq:identities_on_F}, because $\mathcal{K}$ is smooth and strictly convex, $F_{\mathcal{K}}$ (resp. $F^*_{\mathcal{K}}$) is differentiable on $\text{Int}(\mathcal{K})$ (resp. $\mathbb{R}^n$).
Besides, we have  $\mathcal{D}_{\mathcal{K}}^*=\mathbb{R}^n$.
Finally, the strict convexity of $F$ comes from the strict convexity of $\|\cdot\|_{\mathcal{K}}$ when $\mathcal{K}$ is strictly convex.
Hence $F$ is Legendre.
\end{proof}

\subsection{Main Result}\label{ssec:main_results}
Although uniform convexity subsumes strong convexity, for the sake of clarity, we first state in Theorem \ref{th:linear_bandit_strongly_convex} the pseudo-regret upper bounds of Algorithm \ref{algo:bandit_mirror_descent} when the set is strongly convex.
In Theorem \ref{th:linear_bandit_UC}, we then extend these convergence results to the case where the action set is more generally uniformly convex.

\begin{theorem}[Linear Bandit on Strongly Convex Set]\label{th:linear_bandit_strongly_convex}
Consider a compact convex set $\mathcal{K}$ that is centrally symmetric with non-empty interior.
Assume $\mathcal{K}$ is smooth and $\alpha$-strongly convex set w.r.t.~$\|\cdot\|_{\mathcal{K}}$ and $\ell_2(r)\subset \mathcal{K} \subset \ell_{\infty}(R)$ for some $r,R>0$.
Consider running BMD (Algorithm \ref{algo:bandit_mirror_descent}) with the barrier function $F_{\mathcal{K}}(x)=-\ln\big(1-\|x\|_{\mathcal{K}}\big) - \|x\|_{\mathcal{K}}$, and
\begin{equation}\label{eq:parameter_SC}
\eta=\frac{1}{\sqrt{nT}}, ~ \gamma=\frac{1}{\sqrt{T}}.
\end{equation}
For $T\geq 4n\big(\frac{R}{r}\big)^2$ we then have
\begin{equation}\label{eq:pseudo_regret_guarantee}\tag{Pseudo-Regret Upper-Bound}
\Bar{R}_T \leq \sqrt{T} + \sqrt{nT}\ln(T)/2 + L\sqrt{nT} = \tilde{\mathcal{O}}(\sqrt{nT}),
\end{equation}
where $\Bar{R}_T$ is defined in \eqref{eq:pseudo_regret} and $L=(R/r)^2(5\alpha + 4)/\alpha$.
\end{theorem}
\begin{proof}[Proof of Theorem \ref{th:linear_bandit_strongly_convex}]
First note that with $T\geq 4n(R/r)^2$ and $\eta=1/\sqrt{nT}$, we have that $\eta \leq r/(2Rn)$ which allows to invoke Lemma \ref{lem:upper_bound_single_term}.
The proof follows that of \parencite[Theorem 5.8]{bubeck2012regret} but importantly leverages on our novel Lemma \ref{lem:upper_bound_single_term} that carefully upper bounds the terms $D_{F_{\mathcal{K}}^*}(\nabla F_{\mathcal{K}}(x_t)-\eta\tilde{c}_t,\nabla F_{\mathcal{K}}(x_t))$ for the barrier function we designed.
Because $F_{\mathcal{K}}$ is Legendre and $\tilde{c}_t$ is an unbiased estimate of $c_t$, by \parencite[Theorem 5.5]{bubeck2012regret} applied on $\mathcal{K}^\prime \triangleq (1-\gamma)\mathcal{K}$, we have
\[
\Bar{R}_T(\mathcal{K}^\prime) \leq \frac{\text{sup}_{x\in(1-\gamma)\mathcal{K}}F_{\mathcal{K}}(x)  - F_{\mathcal{K}}(x_1)}{\eta} + \frac{1}{\eta} \sum_{t=1}^T\mathbb{E} \Big[D_{F_{\mathcal{K}}^*}\big(\nabla F_{\mathcal{K}}(x_t) - \eta \tilde{c}_t,\nabla F_{\mathcal{K}}(x_t)\big)\Big].
\]
Also, by definition of the \ref{eq:pseudo_regret}, we have
\[
\Bar{R}_T(\mathcal{K}) = \Bar{R}_T(\mathcal{K}^\prime) + \underset{a\in\mathcal{K}^\prime}{\text{min }}\sum_{i=1}^{T} \langle c_t ; a\rangle -  \underset{a\in\mathcal{K}}{\text{min }}\sum_{i=1}^{T} \langle c_t ; a\rangle.
\]
Write $a^*\in\mathcal{K}$ for which $\text{min}_{a\in\mathcal{K}}\sum_{i=1}^{T} \langle c_t ; a\rangle$ is attained.
We have that the $\text{min}_{a\in\mathcal{K}^\prime}\sum_{i=1}^{T} \langle c_t ; a\rangle$ is attained at $(1-\gamma)a^*$, hence because $|\langle c_t; a^*\rangle|\leq 1 $ for any $t$, we have
\[
\Bar{R}_T(\mathcal{K}) = \Bar{R}_T(\mathcal{K}^\prime) + \sum_{i=1}^{T} \langle c_t ; (1-\gamma)a^*\rangle - \sum_{i=1}^{T} \langle c_t ; a^*\rangle = \Bar{R}_T(\mathcal{K}^\prime) - \gamma\sum_{i=1}^{T} \langle c_t ; a^*\rangle \leq \Bar{R}_T(\mathcal{K}^\prime) + \gamma T.
\]
By the initialization of $x_1$ in Line \ref{line:initialization} of Algorithm \ref{algo:bandit_mirror_descent}, we have $F_{\mathcal{K}}(x_1)=0$.
Besides, by definition of $F_{\mathcal{K}}$, $\text{sup}_{x\in\mathcal{K}}F_{\mathcal{K}}(x)\leq \ln(1/\gamma)$, so that $\text{sup}_{x\in\mathcal{K}}F(x) - F_{\mathcal{K}}(x_1)\leq\ln(1/\gamma)$.
Overall, we have
\[
\Bar{R}_T(\mathcal{K}) \leq \gamma T +  \frac{\ln(1/\gamma)}{\eta} + \frac{1}{\eta} \sum_{t=1}^T\mathbb{E} \Big[D_{F_{\mathcal{K}}^*}\big(\nabla F_{\mathcal{K}}(x_t) - \eta \tilde{c}_t,\nabla F_{\mathcal{K}}(x_t)\big)\Big].
\]
% We have $\eta\leq 1/(2n)$ and hence Lemma \ref{lem:upper_bound_single_term} implies that
We have $\eta\leq r/(2Rn)$ and hence Lemma \ref{lem:upper_bound_single_term} implies that
\[
\Bar{R}_T(\mathcal{K}) \leq \gamma T +  \frac{\ln(1/\gamma)}{\eta} + \eta \Big(1+\frac{4(\alpha+1)}{\alpha}\Big) \sum_{t=1}^T\mathbb{E}\Big((1-\|x\|_{\mathcal{K}})\|\tilde{c}_t\|^2_{\mathcal{K^\circ}}\Big).
\]
Then, let us explicit $\mathbb{E}\Big((1-\|x\|_{\mathcal{K}})\|\tilde{c}_t\|^2_{\mathcal{K^\circ}}\Big)$. 
Recall that $\mathcal{K}\subset\ell_{\infty}(R)$, so that $\ell_{\infty}^\circ(R)=\ell_1(1/R)\subset \mathcal{K}^\circ$ and $e_i/R\in\mathcal{K}^\circ$.
Hence, we have that $\|r e_i\|_{\mathcal{K}^\circ} = r R \|e_i/R\|_{\mathcal{K}^\circ}\leq r R$.
We obtain
\begin{eqnarray*}
\mathbb{E}\Big((1-\|x\|_{\mathcal{K}})\|\tilde{c}_t\|^2_{\mathcal{K^\circ}}\Big) &=& \mathbb{P}(\xi_t=0) \sum_{i=1}^{n}\frac{1}{n}(1-\|x_t\|_{\mathcal{K}})\frac{n^2}{r^4} \Big(\frac{\langle r e_i ; c_t \rangle}{1-\|x_t\|_{\mathcal{K}}}\Big)^2 \|r e_i\|_{\mathcal{K}^\circ}^2\\
&\leq& (1-\|x_t\|_{\mathcal{K}}) \sum_{i=1}^{n}n R^2 \frac{c_{t,i}^2}{1-\|x_t\|_{\mathcal{K}}} = n R^2 \|c_t\|_2^2.
\end{eqnarray*}
We have $\ell_2(r)\subset\mathcal{K}$. 
This implies $\mathcal{K}^\circ \subset \ell_2(r)^\circ= \ell_2(1/r)$ so that with $c_t\in\mathcal{K}^\circ$, we have $\|c_t\|_2^2\leq 1/r^2$.
Hence
\[
\Bar{R}_T(\mathcal{K}) \leq \gamma T +  \frac{\ln(1/\gamma)}{\eta} + \eta \Big(1+\frac{4(\alpha+1)}{\alpha}\Big) n \Big(\frac{R}{r}\Big)^2 T,
\]
and we immediately obtain \eqref{eq:pseudo_regret_guarantee} with the prescribed choice of $\eta$ and $\gamma$.
\end{proof}

\begin{theorem}[Linear Bandit on Uniformly Convex Sets]\label{th:linear_bandit_UC}
Let $\alpha>0$, $q\geq 2$, and $p\in]1,2]$ s.t.~$1/p + 1/q=1$.
Consider a compact convex set $\mathcal{K}$ that is centrally symmetric with non-empty interior.
Assume $\mathcal{K}$ is smooth and $(\alpha, q)$-uniformly convex set w.r.t.~$\|\cdot\|_{\mathcal{K}}$ and $\ell_q(r)\subset \mathcal{K} \subset \ell_{\infty}(R)$ for some $r,R>0$.
Consider running BMD (Algorithm \ref{algo:bandit_mirror_descent}) with the barrier function $F_{\mathcal{K}}(x)=-\ln\big(1-\|x\|_{\mathcal{K}}\big) - \|x\|_{\mathcal{K}}$, and
\begin{equation}\label{eq:parameter_UC}
\eta=1/(n^{1/q}T^{1/p}), ~ \gamma=1/\sqrt{T}.
\end{equation}
Then we have for $T\geq 2^p n \big(\frac{R}{r}\big)^p$
\begin{equation}\label{eq:pseudo_regret_guarantee_UC}
\Bar{R}_T \leq \sqrt{T} + n^{1/q} T^{1/p} \ln(T)/2 +  ((1/2)^{2-p} + L) \Big(\frac{R}{r}\Big)^p n^{1/q} T^{1/p} = \tilde{\mathcal{O}}(n^{1/q} T^{1/p}),
\end{equation}
where $\Bar{R}_T$ is defined in \eqref{eq:pseudo_regret} and $L=2p(1 + (q/(2\alpha))^{1/(q-1)})$.
\end{theorem}
\begin{proof}
The proof is similar to Theorem \ref{th:linear_bandit_strongly_convex} and hence to \parencite[Theorem 5.8]{bubeck2012regret}. The difference is that we now leverage Corollary \ref{cor:upper_bound_single_term_UC}.
Note that with $T\geq 2^p n (R/r)^p$ and $\eta=n^{-1/q}T^{-1/p}$, we have $0\leq \eta \leq 1/(2n)(r/R)$.
As in the proof of Theorem \ref{th:linear_bandit_strongly_convex}, we have
\[
\Bar{R}_T(\mathcal{K}) \leq \gamma T +  \frac{\ln(1/\gamma)}{\eta} + \frac{1}{\eta} \sum_{t=1}^T\mathbb{E} \Big[D_{F_{\mathcal{K}}^*}\big(\nabla F_{\mathcal{K}}(x_t) - \eta \tilde{c}_t,\nabla F_{\mathcal{K}}(x_t)\big)\Big].
\]
Now applying Corollary \ref{cor:upper_bound_single_term_UC}, we have with $L=2p(1 + (q/(2\alpha))^{1/(q-1)})$
\[
D_{F_{\mathcal{K}}^*}(\nabla F_{\mathcal{K}}(x_t) - \eta \tilde{c}_t,\nabla F_{\mathcal{K}}(x_t)) \leq  (1-\|x_t\|_{\mathcal{K}}) \eta^p\|\tilde{c}_t\|^p_{\mathcal{K}^\circ}((1/2)^{2-p} + L).
\]
This hence implies
\[
\Bar{R}_T(\mathcal{K}) \leq \gamma T +  \frac{\ln(1/\gamma)}{\eta} + \eta^{p-1} ((1/2)^{2-p} + L)\sum_{t=1}^T\mathbb{E} \Big[(1-\|x\|_{\mathcal{K}})\|\tilde{c}_t\|^p_{\mathcal{K}^\circ}\Big].
\]
Let us now upper bound $\mathbb{E} \Big[(1-\|x\|_{\mathcal{K}})\|\tilde{c}_t\|^p_{\mathcal{K}^\circ}\Big]$.
Since $\mathcal{K}\subset\ell_{\infty}(R)$, we have $\ell_1(1/R)\subset\mathcal{K}^\circ$ and $e_i/R\in\mathcal{K}^\circ$ so that $\|r e_i\|_{\mathcal{K}^\circ}\leq rR$.
Hence, we have
\begin{eqnarray*}
\mathbb{E}\Big((1-\|x\|_{\mathcal{K}})\|\tilde{c}_t\|^p_{\mathcal{K^\circ}}\Big) &=& \mathbb{P}(\xi_t=0) \sum_{i=1}^{n}\frac{1}{n}(1-\|x_t\|_{\mathcal{K}})\Big(\frac{n}{r^2}\Big)^p \Big(\frac{|\langle r e_i; c_t \rangle|}{1-\|x_t\|_{\mathcal{K}}}\Big)^p \|re_i\|_{\mathcal{K}^\circ}^p\\
&\leq& (1-\|x_t\|_{\mathcal{K}})^{2-p} \sum_{i=1}^{n}n^{p-1} R^p c_{t,i}^p \leq n^{p-1} R^p \|c_t\|_p^p.
\end{eqnarray*}
Then since $\ell_q(r)\subset\mathcal{K}$, we have $\mathcal{K}^\circ\subset\ell_q(r)^\circ=\ell_p(1/r)$ so that $\|c_t\|_p\leq 1/r$ because $c_t\in\mathcal{K}^\circ$.
We ultimately obtain
\[
\Bar{R}_T(\mathcal{K}) \leq \gamma T +  \frac{\ln(1/\gamma)}{\eta} + \eta^{p-1} ((1/2)^{2-p} + L)T n^{p-1} \Big(\frac{R}{r}\Big)^p.
\]
Here, we choose $\eta$ of the form $T^{-\beta}n^{-\nu}$ with $\beta$ and $\nu$ such that the terms $1/\eta$ and $\eta^{p-1}Tn^{p-1}$ exhibit the same asymptotic rate in $n$ and $T$ respectively.
In particular, we choose $\eta=1/(n^{1/q}T^{1/p})$ and obtain (with $\gamma=1/\sqrt{T}$)
\[
\Bar{R}_T(\mathcal{K}) \leq \sqrt{T} + n^{1/q} T^{1/p} \ln(T)/2 +  ((1/2)^{2-p} + L) \Big(\frac{R}{r}\Big)^p n^{1/q} T^{1/p}.
\]
\end{proof}

Instantiating the regret bound in Theorem \ref{th:linear_bandit_UC} with $p=q=2$ results in the same regret bound as in Theorem \ref{th:linear_bandit_strongly_convex}. 
Indeed, the parameters in \eqref{eq:parameter_UC} with $q=2$ correspond to \eqref{eq:parameter_SC}.

\begin{remark}
Consider two compact convex sets $\mathcal{K}_1$ and $\mathcal{K}_2$. 
Their \emph{relative width} is defined as follows
\begin{equation}\label{eq:relative_width}\tag{Relative-Width}
w(\mathcal{K}_1,\mathcal{K}_2) \triangleq \underset{x\in\mathcal{K}_1, y\in\mathcal{K}_2}{\text{sup }} \langle x ; y \rangle.
\end{equation}
Note that $w(\mathcal{K}_1,\mathcal{K}_2)= \text{sup}_{x\in\mathcal{K}_1}\|x\|_{\mathcal{K}_2^\circ}$ and $\ell_q(r)^\circ = \ell_p(1/r)$, using the \eqref{eq:relative_width} we could replace the condition $\ell_q(r)\subset\mathcal{K}$ by $w(\mathcal{K}^\circ,\ell_q(1))\leq 1/r$.
\end{remark}

\subsection{Technical Lemmas}\label{ssec:technical_lemma}

We now detail the lemmas invoked in the proofs of Theorems \ref{th:linear_bandit_strongly_convex} and \ref{th:linear_bandit_UC}.
Lemma \ref{eq:identities_on_F} provides the expression for $\nabla F_{\mathcal{K}}$ and $\nabla F_{\mathcal{K}}^*$ and their differentiability domain.
Lemma \ref{lem:lower_bound_theta} is a technicality that notably explains why we constrain $\eta$ in $[0,r/(2nR)]$.
Lemma \ref{lem:upper_bound_single_term} (resp. Corollary \ref{cor:upper_bound_single_term_UC}) are instrumental in upper-bounding the terms $D_{F_{\mathcal{K}}^*}(\nabla F_{\mathcal{K}}(x_t) - \eta \tilde{c}_t,\nabla F_{\mathcal{K}}(x_t))$ when the set is strongly convex (resp. uniformly convex).
Technically, we build the link between the uniform convexity of the set and upper bounds on the regret in these lemmas.
Although uniform convexity is a weaker assumption than strong convexity, we distinguish the cases to stress the convergence results when the action sets are strongly convex.
All lemmas are self-contained and stated independently from Algorithm \ref{algo:bandit_mirror_descent}.

\begin{lemma}[Some Identities]\label{eq:identities_on_F}
Assume $\mathcal{K}\subset\mathbb{R}^n$ is strictly convex compact and smooth set.
Let $x\in\mathcal{K}$ s.t.~$\|x\|_{\mathcal{K}}<1$ and $d\in\mathbb{R}^n\setminus\{0\}$.
With $F_{\mathcal{K}}(x)=-\ln\big(1-\|x\|_{\mathcal{K}}\big) - \|x\|_{\mathcal{K}}$, $F_{\mathcal{K}}$ (resp. $F^*_{\mathcal{K}}$) is differentiable on $\text{Int}(\mathcal{K})$ (resp. $\mathbb{R}^n$) and 
we have
\begin{equation}\label{eq:identities_F_F_star}
  \left\{
    \begin{split}
    & \nabla F_{\mathcal{K}}(x) = \frac{\|x\|_{\mathcal{K}}}{1 - \|x\|_{\mathcal{K}}} \nabla \|\cdot\|_{\mathcal{K}}(x)\\
    & F_{\mathcal{K}}^*(d) = \|d\|_{\mathcal{K}^\circ} - \ln(1 + \|d\|_{\mathcal{K}^\circ})\\
    & \nabla F_{\mathcal{K}}^*(d) = \frac{\|d\|_{\mathcal{K}^\circ}}{1+\|d\|_{\mathcal{K}^\circ}} \nabla \|\cdot\|_{\mathcal{K}^\circ}(d).
    \end{split}
  \right.
\end{equation}
\end{lemma}
\begin{proof}
Let us first compute $F_{\mathcal{K}}^*$.
We have $F_{\mathcal{K}}(d)=g\circ \|x\|_{\mathcal{K}}$ with $g(r)= -\ln(1-r) - r$ for $r\in[0,1[$.
Note that $g(0)=0$ and $g$ is convex.
Write $g^*(y)\triangleq \text{sup}_{r\in [0,1]} yr +\ln(1-r) +r$ for $y\geq 0$.
With simple analysis, we have $g^*(y)= y -\ln(1+y)$.
Then, with, \textit{e.g.}, \parencite[1.47]{schneider2014convex}, we have that $F^*_{\mathcal{K}}(d) = g^*\circ \|d\|_{\mathcal{K}^\circ} = \|d\|_{\mathcal{K}^\circ} -\ln(1 + \|d\|_{\mathcal{K}^\circ})$.

The gradient identities \eqref{eq:identities_F_F_star} are then immediate at points $(x,d)$ s.t.~$\|\cdot\|_{\mathcal{K}}$ and $\|\cdot\|_{\mathcal{K}^\circ}$ are differentiable.
From Lemma \ref{lem:duality_smoothness_convexity} since $\mathcal{K}$ is smooth, $\mathcal{K}^\circ$ is strictly convex.
For $(x,d)\in\mathcal{K}\setminus\{0\}\times \mathbb{R}^n\setminus\{0\}$, by Lemma \ref{eq:differentiability_gauge} \ref{item:set_smooth_differentiable}, we have that $\|\cdot\|_{\mathcal{K}}$ and $\|\cdot\|_{\mathcal{K}^\circ}$ are differentiable.
$F_{\mathcal{K}}$ and $F_{\mathcal{K}^\circ}$ are then also differentiable at $\{0\}$ because $\|\nabla F_{\mathcal{K}}(x)\|$ and $\|\nabla F_{\mathcal{K}^\circ}(d)\|$ converges to zero as $x$ and $d$ converge to zero (since $\nabla\|\cdot\|_{\mathcal{K}}(x)$ is of norm one).
\end{proof}

\begin{lemma}[Lower Bound on $\Theta$]\label{lem:lower_bound_theta}
Assume $\ell_1(r)\subset\mathcal{K}\subset\ell_{\infty}(R)$ for some $r,R>0$.
Let $x\in\mathcal{K}$ with $\|x\|_{\mathcal{K}}<1$, $\eta>0$ and $c\in\mathcal{K}^\circ$.
Consider the realizations of random variable $\xi\sim \text{Ber}(\|x\|_{\mathcal{K}})$, $i\sim \frac{1}{n}\mathbbm{1}_n$, and $\epsilon\sim\text{Rad}(\frac{1}{2})$. 
We define $a\in\mathcal{K}$ (resp. $\tilde{c}$) similarly to $a_t$ (resp. $\tilde{c}_t$) in Algorithm \ref{algo:bandit_mirror_descent} with
    \begin{equation}\label{eq:def_a_c_independent}
     a = \left\{
        \begin{split}
        & x/\|x\|_{\mathcal{K}} \text{ if } \xi=1\\
        & r \epsilon~e_{i} \text{ otherwise,}
        \end{split}
      \right.
    ~~\text{ and } ~~\tilde{c}= \frac{n}{r^2} (1-\xi) \frac{\langle a; c\rangle}{1 - \|x\|_{\mathcal{K}}}a.
    \end{equation}
Write $u=\nabla F_{\mathcal{K}}(x) - \eta \tilde{c}$ and  $v=\nabla F_{\mathcal{K}}(x)$.
Then, we have
\begin{equation}\label{eq:lower_bound_theta}
\frac{\|u\|_{\mathcal{K}^\circ} -\|v\|_{\mathcal{K}^\circ}}{1+\|v\|_{\mathcal{K}^\circ}}\geq -\eta n \frac{R}{r}.
\end{equation}
\end{lemma}
\begin{proof}
Note that because $\ell_{1}(r)\subset\mathcal{K}$, we have $\pm r e_i\in\mathcal{K}$ and in particular $a\in\mathcal{K}$.
We now follow the argument of \parencite{bubeck2012regret}.
With the expression of $\nabla F_{\mathcal{K}}(x)$ in Lemma \ref{eq:identities_on_F} and that $\big\|\nabla\|\cdot\|_{\mathcal{K}}(x)\big\|_{\mathcal{K}^\circ} = 1$ in Lemma \ref{eq:differentiability_gauge}, we have $\frac{1}{1+\|\nabla F_{\mathcal{K}}(x)\|_{\mathcal{K}^\circ}} = 1 - \|x\|_{\mathcal{K}}$.
So with the triangle inequality, we have  $\|v - \eta \tilde{c}\|_{\mathcal{K}^\circ}\geq \|v\|_{\mathcal{K}^\circ} - \eta \|\tilde{c}\|_{\mathcal{K}^\circ}$ so that we obtain
\[
\frac{\|v - \eta \tilde{c}\|_{\mathcal{K}^\circ} -\|v\|_{\mathcal{K}^\circ}}{1+\|v\|_{\mathcal{K}^\circ}}\geq -\eta\|\tilde{c}\|_{\mathcal{K}^\circ}(1 - \|x\|_{\mathcal{K}}).
\]
Then, since $\mathlarger{\tilde{c}=\frac{n}{r^2}(1-\xi)\frac{\langle a; c\rangle}{1- \|x\|_{\mathcal{K}}}a}$, we have
\begin{equation}\label{eq:intermediary}
\frac{\|v - \eta \tilde{c}\|_{\mathcal{K}^\circ} -\|v\|_{\mathcal{K}^\circ}}{1+\|v\|_{\mathcal{K}^\circ}}\geq -\eta \frac{n}{r^2} (1-\xi) |\langle a;c\rangle| \cdot \|a\|_{\mathcal{K}^\circ}.
\end{equation}
Because $\|\cdot\|_{\mathcal{K}}$ and $\|\cdot\|_{\mathcal{K}^\circ}$ are dual norms and $(a,c)\in\mathcal{K}\times\mathcal{K}^\circ$ we have
$|\langle a; c\rangle| \leq \|a\|_{\mathcal{K}} \|c\|_{\mathcal{K}^\circ} \leq 1$, which leads to
\[
\frac{\|v - \eta \tilde{c}\|_{\mathcal{K}^\circ} -\|v\|_{\mathcal{K}^\circ}}{1+\|v\|_{\mathcal{K}^\circ}}\geq -\eta \frac{n}{r^2}\|a\|_{\mathcal{K}^\circ}.
\]
When $\xi=1$, \eqref{eq:lower_bound_theta} is already satisfied. 
Otherwise, $\xi=0$ and by definition of $a$, we have $a=\epsilon r e_i$ with $i\in[n]$ and $\epsilon\in\{-1, 1\}$.
Since $\mathcal{K}\subset\ell_{\infty}(R)$, we have $\ell_{\infty}(R)^\circ =\ell_1(1/R) \subset \mathcal{K}^\circ$ and $e_i/R\in\mathcal{K}^\circ$.
Hence, $\|r e_i\|_{\mathcal{K}^\circ} = rR \| e_i/R\|_{\mathcal{K}^\circ}\leq rR$.
So finally, we obtain
\[
\frac{\|v - \eta \tilde{c}\|_{\mathcal{K}^\circ} -\|v\|_{\mathcal{K}^\circ}}{1+\|v\|_{\mathcal{K}^\circ}}\geq -\eta n \frac{R}{r}.
\]
\end{proof}

The following lemma is instrumental to obtaining the pseudo-regret bounds.
Note that the distance of $x_t$ to $\mathcal{K}$ is controlled by $\gamma$, see Line \ref{line:mirror_descent_step} in Algorithm \ref{algo:bandit_mirror_descent}.
Finally, the sole difference with the bound obtained with the Euclidean ball is with the extra factor $1 + 4(\alpha + 1)/\alpha$ and the constraint in $\eta$ that now depends on the ratio $r/R$ which, \textit{e.g.}, equals $1$ for any $\ell_q(1)$ ball.

\begin{lemma}[One Term Upper Bound Strong Convexity]\label{lem:upper_bound_single_term}
Consider $\mathcal{K}$ a $\alpha$-strongly convex and centrally symmetric set with non-empty interior. 
Assume that $\ell_1(r)\subset\mathcal{K}\subset\ell_{\infty}(R)$ for some $r,R>0$.
Let $x\in\mathcal{K}$ s.t.~$\|x\|_{\mathcal{K}}<1$ and $\tilde{c}$ as defined in \eqref{eq:def_a_c_independent}.
If $0<\eta\leq \frac{1}{2n} \frac{r}{R}$, then we have
\begin{equation}\label{eq:major_upper_bound_strongly_convex}
D_{F_{\mathcal{K}}^*}(\nabla F_{\mathcal{K}}(x) - \eta \tilde{c},\nabla F_{\mathcal{K}}(x)) \leq  (1-\|x\|_{\mathcal{K}}) \Big(1 + \frac{4(\alpha + 1)}{\alpha}\Big)\eta^2\|\tilde{c}\|^2_{\mathcal{K}^\circ}.
\end{equation}
\end{lemma}
\begin{proof}[Proof of Lemma \ref{lem:upper_bound_single_term}]
Let us write $u=\nabla F_{\mathcal{K}}(x) - \eta \tilde{c}$, $v=\nabla F_{\mathcal{K}}(x)$, and $\Theta = \frac{\|u\|_{\mathcal{K}^\circ} -\|v\|_{\mathcal{K}^\circ}}{1+\|v\|_{\mathcal{K}^\circ}}$.
Elementary manipulations combined with Lemma \ref{eq:identities_on_F} give
\begin{eqnarray*}
D_{F_{\mathcal{K}}^*}(u,v) &=& F_{\mathcal{K}}^*(u) - F_{\mathcal{K}}^*(v) - \langle \nabla F_{\mathcal{K}}^*(v);u-v\rangle\\
&=& \|u\|_{\mathcal{K}^\circ} -\|v\|_{\mathcal{K}^\circ} - \ln\Big(\frac{1 + \|u\|_{\mathcal{K}^\circ}}{1 + \|v\|_{\mathcal{K}^\circ}}\Big) - \frac{\|v\|_{\mathcal{K}^\circ}}{1+\|v\|_{\mathcal{K}^\circ}} \big\langle \nabla \|\cdot\|_{\mathcal{K}^\circ}(v); u- v\big\rangle\\
&=& \|u\|_{\mathcal{K}^\circ} -\|v\|_{\mathcal{K}^\circ} - \ln\big(1 +\Theta\big) - \frac{\|v\|_{\mathcal{K}^\circ}}{1+\|v\|_{\mathcal{K}^\circ}} \big\langle \nabla \|\cdot\|_{\mathcal{K}^\circ}(v); u- v\big\rangle\\
&=& \frac{1}{1+\|v\|_{\mathcal{K}^\circ}}\Big[ (1+\|v\|_{\mathcal{K}^\circ})(\|u\|_{\mathcal{K}^\circ} -\|v\|_{\mathcal{K}^\circ}) - (1+\|v\|_{\mathcal{K}^\circ})\ln\big(1 +\Theta\big) \\
& & \qquad \qquad \qquad \qquad \qquad \qquad \qquad \qquad \qquad- \|v\|_{\mathcal{K}^\circ} \big\langle \nabla \|\cdot\|_{\mathcal{K}^\circ}(v); u- v\big\rangle\Big]\\
&=& \Theta - \ln\big(1 +\Theta \big) + \frac{1}{1+\|v\|_{\mathcal{K}^\circ}}\underbrace{\Big[ \|v\|_{\mathcal{K}^\circ}(\|u\|_{\mathcal{K}^\circ} -\|v\|_{\mathcal{K}^\circ}) - \|v\|_{\mathcal{K}^\circ} \big\langle \nabla \|\cdot\|_{\mathcal{K}^\circ}(v); u- v\big\rangle\Big]}_{H \triangleq}.
\end{eqnarray*}
Let us add and subtract $-\frac{1}{2}\|u\|^2_{\mathcal{K}^\circ}$ in $H$. We obtain
\[
H = \|v\|_{\mathcal{K}^\circ}\|u\|_{\mathcal{K}^\circ} - \frac{1}{2}\|v\|^2_{\mathcal{K}^\circ} - \frac{1}{2}\|u\|^2_{\mathcal{K}^\circ} + \frac{1}{2} \|u\|^2_{\mathcal{K}^\circ} -\frac{1}{2}\|v\|^2_{\mathcal{K}^\circ} - \big\langle \|v\|_{\mathcal{K}^\circ}\nabla \|\cdot\|_{\mathcal{K}^\circ}(v); u- v\big\rangle.
\]
We note that $\nabla \frac{1}{2}\|\cdot\|^2_{\mathcal{K}^\circ}(v) = \|v\|_{\mathcal{K}^\circ} \nabla \|\cdot\|_{\mathcal{K}^\circ}(v)$. It is then crucial to observe that the Bregman divergence of $\frac{1}{2}\|\cdot\|_{\mathcal{K}^\circ}$ appears as follows
\begin{eqnarray*}
H & = & \|v\|_{\mathcal{K}^\circ}\|u\|_{\mathcal{K}^\circ} - \frac{1}{2}\|v\|^2_{\mathcal{K}^\circ} - \frac{1}{2}\|u\|^2_{\mathcal{K}^\circ} + D_{\frac{1}{2}\|\cdot\|^2_{\mathcal{K}^\circ}}(u,v)\\
& = & -\frac{1}{2} \big(\|u\|_{\mathcal{K}^\circ} - \|v\|_{\mathcal{K}^\circ}\big)^2 + D_{\frac{1}{2}\|\cdot\|^2_{\mathcal{K}^\circ}}(u,v).
\end{eqnarray*}
Overall, with careful rewriting, we obtain that for any $(u,v)\in\mathbb{R}^n$
\[
D_{F_{\mathcal{K}}^*}(u,v) = \Theta - \ln\big(1 + \Theta \big)  -\frac{1}{2} \frac{\big(\|u\|_{\mathcal{K}^\circ} - \|v\|_{\mathcal{K}^\circ}\big)^2}{1+\|v\|_{\mathcal{K}^\circ}} + \frac{1}{1+\|v\|_{\mathcal{K}^\circ}} D_{\frac{1}{2}\|\cdot\|^2_{\mathcal{K}^\circ}}(u,v).
\]
With $\frac{1}{1+\|v\|_{\mathcal{K}^\circ}} = \frac{1}{1+\|\nabla F_{\mathcal{K}}(x)\|_{\mathcal{K}^\circ}} = 1 - \|x\|_{\mathcal{K}}$ (Lemma \ref{eq:identities_on_F} and $\nabla \|\cdot\|_{\mathcal{K}}(x)$ is norm $1$) it follows
\[
D_{F_{\mathcal{K}}^*}(u,v) \leq \Theta - \ln\big(1 + \Theta \big) + (1 - \|x\|_{\mathcal{K}}) D_{\frac{1}{2}\|\cdot\|^2_{\mathcal{K}^\circ}}(u,v).
\]
Then, to upper bound $\Theta - \ln\big(1 + \Theta \big)$, we note that $\ln(1+\theta)\geq \theta - \theta^2$ for all $\theta\geq -\frac{1}{2}$. 
Hence, we need to choose $\eta$ such that $\Theta\geq -\frac{1}{2}$.
If $-\eta n \frac{R}{r}\geq -\frac{1}{2}$, \textit{i.e.}, for $\eta\leq \frac{1}{2n} \frac{r}{R}$, Lemma \ref{lem:lower_bound_theta} implies that $\Theta\geq -\frac{1}{2}$. 
Thus,
\[
D_{F_{\mathcal{K}}^*}(u,v) \leq \Big(\frac{\|u\|_{\mathcal{K}^\circ} -\|v\|_{\mathcal{K}^\circ}}{1+\|v\|_{\mathcal{K}^\circ}}\Big)^2 + (1-\|x\|_{\mathcal{K}}) D_{\frac{1}{2}\|\cdot\|^2_{\mathcal{K}^\circ}}(u,v).
\]
Then, by the triangle inequality, and $1/(1+\|v\|_{\mathcal{K}^\circ})= 1 - \|x\|_{\mathcal{K}}$, we have
\begin{equation}\label{eq:central_intermediary_upper_bound}
D_{F_{\mathcal{K}}^*}(u,v) \leq (1-\|x\|_{\mathcal{K}})^2\|u-v\|^2_{\mathcal{K}^\circ} + (1-\|x\|_{\mathcal{K}}) D_{\frac{1}{2}\|\cdot\|^2_{\mathcal{K}^\circ}}(u,v).
\end{equation}
Then, with Corollary \ref{cor:bregman_upper_bound_strong_convexity}, we have
$D_{\frac{1}{2}\|\cdot\|^2_{\mathcal{K}^\circ}}(u,v)\leq \frac{4(\alpha +1)}{\alpha} \|u-v\|^2_{\mathcal{K}^\circ}$.
Hence by combining it with \eqref{eq:central_intermediary_upper_bound}, we obtain
\[
D_{F_{\mathcal{K}}^*}(u,v) \leq (1-\|x\|_{\mathcal{K}})\|u-v\|^2_{\mathcal{K}^\circ} \big[ 1 + \frac{4(\alpha + 1)}{\alpha}\big].
\]
\end{proof}

With the very same technique, we obtain another form of upper bound when the set is uniformly convex. 
For the sake of clarity we write it as a corollary of Lemma \ref{lem:upper_bound_single_term} although it is an extension.

\begin{corollary}[One Term Upper Bound Uniform Convexity]\label{cor:upper_bound_single_term_UC}
Let $q\geq 2$ and $p\in ]1,2]$ s.t.~$1/p + 1/q =1$.
Consider $\mathcal{K}$ an $(\alpha,q)$-uniformly convex and centrally symmetric with non-empty interior set.
Assume that $\ell_1(r)\subset\mathcal{K}\subset\ell_\infty(R)$ for some $r,R>0$.
Let $x\in\mathcal{K}$ s.t.~$\|x\|_{\mathcal{K}}<1$ and $\tilde{c}$ as defined in \eqref{eq:def_a_c_independent}.
If $0<\eta\leq \frac{1}{2n}\frac{r}{R}$, then we have
\begin{equation}
D_{F_{\mathcal{K}}^*}(\nabla F_{\mathcal{K}}(x) - \eta \tilde{c},\nabla F_{\mathcal{K}}(x)) \leq  (1-\|x\|_{\mathcal{K}}) \eta^p\|\tilde{c}\|^p_{\mathcal{K}^\circ}((1/2)^{2-p} + L),
\end{equation}
with $L\triangleq 2p(1 + (q/(2\alpha))^{1/(q-1)})$.
\end{corollary}

\begin{proof}[Proof of Corollary \ref{cor:upper_bound_single_term_UC}]
The proof is exactly the same as Lemma \ref{lem:upper_bound_single_term} until \eqref{eq:central_intermediary_upper_bound}.
Here, by \eqref{eq:bregman_upper_bound} in Lemma \ref{lem:bregman_upper_bound}, we have $D_{\frac{1}{2}\|\cdot\|^2_{\mathcal{K}^\circ}}(u,v) \leq 2p \big(1 + (q/(2\alpha))^{1/(q-1)}\big) \|u-v\|^{p}_{\mathcal{K}^\circ}$.
Hence, we now have
\begin{eqnarray*}
D_{F_{\mathcal{K}}^*}(u,v) & \leq & (1-\|x\|_{\mathcal{K}})^2\|u-v\|^2_{\mathcal{K}^\circ} + (1-\|x\|_{\mathcal{K}}) 2p\big(1 + (q/(2\alpha))^{1/(p-1)}\big) \|u-v\|^{p}_{\mathcal{K}^\circ}\\
& \leq & (1-\|x\|_{\mathcal{K}})\|u-v\|^p_{\mathcal{K}^\circ}\big[ (1-\|x\|_{\mathcal{K}})\|u-v\|^{2-p}_{\mathcal{K}^\circ} + 2p\big(1 + (q/(2\alpha))^{1/(q-1)}\big)\big].
\end{eqnarray*}
We now simply need to bound the term $(1-\|x\|_{\mathcal{K}})\|u-v\|^{2-p}_{\mathcal{K}^\circ}$. 
We have $u-v = \eta \tilde{c}$, and by definition of $\tilde{c}$ in \eqref{eq:def_a_c_independent}, when $\xi=0$, we have
\[
(1-\|x\|_{\mathcal{K}})\|u-v\|^{2-p}_{\mathcal{K}^\circ} = (1-\|x\|_{\mathcal{K}})^{p-1}\Big[\frac{n\eta}{r^2}|\langle c; re_i\rangle | \cdot\|re_i\|_{\mathcal{K}^\circ}\Big]^{2-p}.
\]
Then, since $\ell_1(r)\subset\mathcal{K}$, $re_i\in\mathcal{K}$ and $c\in\mathcal{K}^\circ$, we have $|\langle c; re_i\rangle|\leq 1$.
Also, since $\mathcal{K}\subset\ell_{\infty}(R)$, we have $\ell_{\infty}(R)^\circ=\ell_1(1/R)\subset\mathcal{K}^\circ$ and $e_i/R\in\mathcal{K}^\circ$, hence $\|re_i\|_{\mathcal{K}^\circ}\leq rR$.
Besides, by the choice of $\eta$, we have $n\eta\leq r/(2R)$.
We now have (case $\xi=1$ is immediate) with $\eta\leq r/(2nR)$ and because $(1-\|x\|_{\mathcal{K}})\leq 1$ and $p-1>0$
\[
(1-\|x\|_{\mathcal{K}})\|u-v\|^{2-p}_{\mathcal{K}^\circ} \leq 1\cdot\Big[n\eta \frac{rR}{r^2}\Big]^{2-p} \leq \Big[\frac{r}{2R} \frac{rR}{r^2}\Big]^{2-p} = 1/2^{2-p}.
\]
Finally, we obtain
\[
D_{F_{\mathcal{K}}^*}(u,v)  \leq  (1-\|x\|_{\mathcal{K}})\|u-v\|^p_{\mathcal{K}^\circ}\big[ (1/2)^{2-p} + 2p\big(1 + (q/(2\alpha))^{1/(q-1)}\big)\big].
\]
\end{proof}

\section{Conclusion}
When the action set is strongly convex, we design a barrier function leading to a bandit algorithm with pseudo-regret in $\tilde{\mathcal{O}}(\sqrt{nT})$.
We hence drastically extend the family of action sets for which such pseudo-regret hold, answering an open question of \parencite{bubeck2012regret}.
To our knowledge, a $\tilde{\mathcal{O}}(\sqrt{nT})$ bound was known only when the action set is a simplex or an $\ell_p$ ball with $p\in]1,2]$.
We are now interested in 1) providing lower-bound on the pseudo-regret bounds for strongly convex sets, 2) providing expected or high-probability regret bounds, 3) providing such guarantees in the \textit{starved bandit} setting \parencite{bubeck2018sparsity}.

When the set is $(\alpha, q)$-uniformly convex with $q\geq 2$, in Theorems \ref{th:linear_bandit_strongly_convex} and \ref{th:linear_bandit_UC} we assume that $\ell_q(r)$ is contained in the action set $\mathcal{K}$. 
It is restrictive but allows us to first prove improved pseudo-regret bounds outside the explicit $\ell_p$ case.
Removing this assumption is an interesting research direction. 
However, it is not clear that the current classical algorithmic scheme with a barrier function is best adapted to leverage the strong convexity of the action set. 
Indeed, in the case of online linear learning, \cite{huang2017following} show that the simple FTL allows obtaining accelerated regret bounds. 
Such projection-free schemes have several benefits, \textit{e.g.}, computational efficiency \parencite{combettes2021complexity} but in the case of FTL they also do not require smoothness of the action set \parencite{Molinaro2020} as opposed to Algorithm \ref{algo:bandit_mirror_descent} which requires it to ensure differentiability of $F_{\mathcal{K}}$ and $F_{\mathcal{K}^\circ}$ simultaneously.
Besides, they also exhibit adaptive properties to unknown structural assumptions, \textit{e.g.}, unknown parameters of H\"olderian Error Bounds \parencite{kerdreux2019restarting,kerdreux2020accelerating}.

At a high level, this work is an example of the favorable dimension-dependency of the sets' uniform convexity assumptions for the pseudo-regret bounds. 
It is crucial for large-scale machine learning.
Such observations have already been made, \textit{e.g.}, in constrained optimization \parencite{polyak1966existence,demyanov1970,dunn1979rates,kerdreux2020projection,kerdreux2020affine}, when the sets' $\alpha$-strong convexity leads to linear convergence rates of the Frank-Wolfe methods with a conditioning on the set that does not depend on the dimension. 
On the contrary, the linear convergence regimes for \emph{corrective} versions of Frank-Wolfe on polytope with strongly convex functions suffer large dimension dependency, see, \textit{e.g.}, \parencite{lacoste2015global,diakonikolas2020locally,garber2020revisiting,Carderera2021}.
This difference between polytope structures and uniform convexity assumption is even more apparent with infinite-dimensional constraints.
Besides, to our knowledge, the uniform convexity structures for the sets are much less developed and understood than their functional counterpart, see, \textit{e.g.}, \parencite{kerdreux21}. 
Arguably, this stems from a tendency in machine learning to consider that constraints are theoretically interchangeable with penalization.
It is often not quite accurate in terms of convergence results and the algorithmic strategies developed differ.
The linear bandit setting is a simple example where such symmetry is structurally not relevant.

\onecolumn

\paragraph{Acknowledgements.}
Research reported in this paper was partially supported through the Research Campus Modal funded by the German Federal Ministry of Education and Research (fund numbers 05M14ZAM,05M20ZBM) as well as the Deutsche Forschungsgemeinschaft (DFG) through the DFG Cluster of Excellence MATH+. 
AA is at the d\'epartement d'informatique de l'\'Ecole Normale Sup\'erieure, UMR CNRS 8548, PSL Research University, 75005 Paris, France, and INRIA. 
AA would like to acknowledge support from the {\em ML and Optimisation} joint research initiative with the {\em fonds AXA pour la recherche} and Kamet Ventures, a Google focused award, as well as funding by the French government under management of Agence Nationale de la Recherche as part of the "Investissements d'avenir" program, reference ANR-19-P3IA-0001 (PRAIRIE 3IA Institute).

\printbibliography

\appendix

\section{Consequences of Set Strong Convexity}\label{app:uc_equivalence_simplified}

We provide here a simplification of \parencite[Theorem 4.1.]{kerdreux21}, see also \parencite{borwein2009uniformly}. 
Let us first recall the \emph{scaling inequality} that provide an equivalent characterization of uniformly convex sets \parencite[Theorem 4.1.]{kerdreux21}.
These inequalities quantify the behavior of the normal cone directions at the boundary of $\mathcal{K}$.
As such, they provide a more geometrical intuition on uniform convex than the algebraic Definition \ref{def:uniform_convexity_set}.
Also, they are useful to prove Theorem \ref{th:UC_support}.

\begin{lemma}[Scaling Inequality]
Let $\alpha>0$ and $q\geq 2$.
Assume $\mathcal{K}$ is $(\alpha, q)$-uniformly convex.
Then, for any $x,y\in\mathcal{K}\times\partial\mathcal{K}$ and $d\in N_{\mathcal{K}}(y)$, we have
\begin{equation}\label{eq:scaling_inequalities}
\langle d; y-x\rangle \geq \frac{\alpha}{q} \|x - y\|_{\mathcal{K}}^q\|d\|_{\mathcal{K}^\circ}.
\end{equation}
\end{lemma}
\begin{proof}
We repeat the proof for completeness. Let $(x,y,d)$ as in the lemma. In particular, $y\in\argmax_{v\in\mathcal{K}}\langle d; v\rangle$.
By optimality of $y$ and uniform convexity of $\mathcal{K}$, for any $\gamma\in]0,1[$ and $z$ with $\|z\|_{\mathcal{K}}\leq 1$ we have
\[
\langle d; y \rangle \geq \langle d ; \gamma x + (1-\gamma)y + \frac{\alpha}{q} \gamma (1-\gamma) \|x-y\|^q_{\mathcal{K}}z\rangle.
\]
After simplification, we obtain for any $\gamma\in]0,1[,z\in\mathcal{K}$
\[
\langle d; y -x\rangle \geq \frac{\alpha}{q} (1-\gamma) \|y-x\|_{\mathcal{K}}^q\langle d; z\rangle.
\]
Hence, by definition of the dual norm of $\|\cdot\|_{\mathcal{K}}$ and $\|\cdot\|_{\mathcal{K}}^\star=\|\cdot\|_{\mathcal{K}^\circ}$, we obtain
\[
\langle d; y -x\rangle \geq \frac{\alpha}{q} \|y-x\|_{\mathcal{K}}^q \|d\|_{\mathcal{K}^\circ}.
\]
\end{proof}

Theorem \ref{th:UC_support} is slightly different from \parencite[Theorem 4.1.]{kerdreux21} because we are interested in the smoothness property of $\frac{1}{2}\|\cdot\|^2_{\mathcal{K}^\circ}$ instead of $\frac{1}{q}\|\cdot\|^q_{\mathcal{K}^\circ}$ when the set $\mathcal{K}$ is $(\alpha,q)$-uniformly convex.
The proof is however very similar.
The main different is that in \parencite[Theorem 4.1.]{kerdreux21} the smoothness property was ensured on $\mathbb{R}^n$ while here it is only true on bounded domains like $\mathcal{K}^\circ$. 
% Overall, Theorem 

\begin{theorem}\label{th:UC_support}
Let $\alpha>0$, $q\geq2$ and $p\in]1,2]$ s.t $1/p + 1/q=1$.
Consider $\mathcal{K}\subset\mathbb{R}^n$ a centrally symmetric compact convex with non-empty interior.
Assume $\mathcal{K}$ is smooth and $(\alpha,q)$-uniformly convex w.r.t.~$\|\cdot\|_{\mathcal{K}}$ (Definition \ref{def:uniform_convexity_set}), then
\begin{equation}
    \frac{1}{2}\|\cdot\|^2_{\mathcal{K}^\circ} \text{ is } (L, p)\text{-H\"older Smooth on }\mathcal{K}^\circ,
\end{equation}
with 
\[
L = 2p\Big(1 + \Big(\frac{q}{2\alpha}\Big)^{1/(q-1)}\Big).
\]
\end{theorem}
\begin{proof}
The proof follows \parencite[Theorem 4.1]{kerdreux21}.
We repeat it to obtain quantitative results. 
The proof proceed is two steps: first prove the H\"older-smoothness of $\|\cdot\|_{\mathcal{K}^\circ}$ on $\partial\mathcal{K}^\circ$ and then prove the H\"older-smoothness of $\frac{1}{2}\|\cdot\|^2_{\mathcal{K}^\circ}$ on $\mathcal{K}^\circ$.

\paragraph{Smoothness of $\|\cdot\|_{\mathcal{K}^\circ}$ on $\partial\mathcal{K}^\circ$.}
Let $(d_1,d_2)\in\partial\mathcal{K}^\circ\times\partial\mathcal{K}^\circ$ and $(x_1,x_2)\in\partial\mathcal{K}\times\partial\mathcal{K}$ s.t.~$x_i\in\text{argmax}_{x\in\mathcal{K}}\langle d_i; x\rangle$ for $i=1,2$.
Because $\mathcal{K}$ is strictly convex (uniform convexity implies strict convexity), the $x_i$ are unique and by Lemma \ref{eq:differentiability_gauge}, $\nabla \|\cdot\|_{\mathcal{K}^\circ}(d_i)=x_i$ for $i=1,2$.
Note that equivalently we have $d_i\in N_{\mathcal{K}}(x_i)$.
Applying the scaling inequalities \eqref{eq:scaling_inequalities} we have for any $x\in\mathcal{K}$
\begin{equation*}
   \left\{
    \begin{split}
    \langle d_1; x_1 - x\rangle &\geq \alpha/q \|d_1\|_{\mathcal{K}^\circ}\cdot \|x_1 - x\|_{\mathcal{K}}^q = \alpha/q \|x_1 - x\|_{\mathcal{K}}^q \\
    \langle d_2; x_2 - x\rangle &\geq \alpha/q \|d_2\|_{\mathcal{K}^\circ}\cdot \|x_2 - x\|_{\mathcal{K}}^q = \alpha/q \|x_2 - x\|_{\mathcal{K}}^q.
    \end{split}
   \right.
\end{equation*}
Then, by summing the two inequalities evaluated respectively at $x=x_2$ and $x=x_1$, we have
\[
\langle d_1 - d_2; x_1 - x_2 \rangle \geq 2\alpha/q\|x_1 - x_2\|_{\mathcal{K}}^q.
\]
By Cauchy-Schwartz, we obtain
\[
\big\|d_1 - d_2\big\|_{\mathcal{K}^\circ} \cdot \big\|\nabla \|\cdot\|_{\mathcal{K}^\circ}(d_1) - \nabla \|\cdot\|_{\mathcal{K}^\circ}(d_2)\big\|_{\mathcal{K}} \geq 2\alpha/q \big\|\nabla \|\cdot\|_{\mathcal{K}^\circ}(d_1) - \nabla \|\cdot\|_{\mathcal{K}^\circ}(d_2)\big\|_{\mathcal{K}}^q,
\]
and conclude that
\begin{equation}\label{eq:smoothness_gauge_sphere}
\big\|\nabla \|\cdot\|_{\mathcal{K}^\circ}(d_1) - \nabla \|\cdot\|_{\mathcal{K}^\circ}(d_2)\big\|_{\mathcal{K}} \leq \frac{1}{(2\alpha/q)^{1/(q-1)}} \big\|d_1 - d_2\big\|_{\mathcal{K}^\circ}^{1/(q-1)}.
\end{equation}

\paragraph{Smoothness of $\frac{1}{2}\|\cdot\|^2_{\mathcal{K}^\circ}$ on $\mathcal{K}^\circ$.}
Let us first note that $\nabla \frac{1}{2}\|\cdot\|^2_{\mathcal{K}^\circ}(d)= \|d\|_{\mathcal{K}^\circ} \nabla \|\cdot\|_{\mathcal{K}^\circ}(d)$.
Hence, since $\|\cdot\|_{\mathcal{K}^\circ}(d)$ is norm $1$, when $d$ approaches $0_{n}$, the limit of $\nabla \frac{1}{2}\|\cdot\|^2_{\mathcal{K}^\circ}(d)$ is $0$ and hence $\frac{1}{2}\|\cdot\|^2_{\mathcal{K}^\circ}$ is differentiable on $\mathbb{R}^n$ (as opposed to $\|\cdot\|_{\mathcal{K}^\circ}$ that is not differentiable at $0$).

Similarly, consider non-zeros $(d_1,d_2)\in\mathcal{K}^\circ\times\mathcal{K}^\circ$ and the $(x_1,x_2)\in\partial\mathcal{K}\times\partial\mathcal{K}$ s.t.~$x_i\in\text{argmax}_{x\in\mathcal{K}}\langle d_i; x\rangle$ for $i=1,2$.
Because of \ref{item:unit_gradient_homogeneity} in Lemma \ref{eq:differentiability_gauge}, we have $\nabla \|\cdot\|_{\mathcal{K}^\circ}(d_1)=\nabla \|\cdot\|_{\mathcal{K}^\circ}(d_1/\|d_1\|_{\mathcal{K}^\circ})$.
Hence, with \eqref{eq:smoothness_gauge_sphere}, we obtain
\[
\big\|\nabla \|\cdot\|_{\mathcal{K}^\circ}(d_1) - \|\cdot\|_{\mathcal{K}^\circ}(d_2)\big\|_{\mathcal{K}} \leq \frac{1}{(2\alpha/q)^{1/(q-1)}} \big\| d_1/\|d_1\|_{\mathcal{K}^\circ} - d_2/\|d_2\|_{\mathcal{K}^\circ}\big\|_{\mathcal{K}^\circ}^{1/(q-1)}.
\]
Write $C \triangleq 1/(2\alpha/q)^{1/(q-1)}$ and $I\triangleq\big\|\nabla \frac{1}{2}\|\cdot\|^2_{\mathcal{K}^\circ}(d_1) - \frac{1}{2}\nabla\|\cdot\|^2_{\mathcal{K}^\circ}(d_2)\big\|_{\mathcal{K}}$. 
Let us now consider
\begin{eqnarray*}
I & = & \big\|\|d_1\|_{\mathcal{K}^\circ}\nabla \|\cdot\|_{\mathcal{K}^\circ}(d_1) -\|d_2\|_{\mathcal{K}^\circ}\nabla \|\cdot\|_{\mathcal{K}^\circ}(d_2)\big\|_{\mathcal{K}}\\
&\leq & \|d_1\|_{\mathcal{K}^\circ} \big\|\nabla \|\cdot\|_{\mathcal{K}^\circ}(d_1) -\nabla\| \cdot\|_{\mathcal{K}^\circ}(d_2) \big\| + \big\|\nabla \|\cdot\|_{\mathcal{K}^\circ}(d_2)\big\| \cdot \big| \|d_1\|_{\mathcal{K}^\circ} - \|d_2\|_{\mathcal{K}^\circ}\big|\\
&\leq & C\|d_1\|_{\mathcal{K}^\circ}^{1 + 1/(q-1)} \|d_1\|_{\mathcal{K}^\circ}^{1/(q-1)} \big\| d_1/\|d_1\|_{\mathcal{K}^\circ} - d_2/\|d_2\|_{\mathcal{K}^\circ}\big\|_{\mathcal{K}^\circ}^{1/(q-1)} + \|d_1 - d_2\|_{\mathcal{K}^\circ}\\
&\leq & C\|d_1\|_{\mathcal{K}^\circ}^{q/(q-1)} \big\| d_1 - d_2\|d_1\|_{\mathcal{K}^\circ}/\|d_2\|_{\mathcal{K}^\circ}\big\|_{\mathcal{K}^\circ}^{1/(q-1)} + \|d_1 - d_2\|_{\mathcal{K}^\circ}^{q/(q-1)} \|d_1 - d_2\|_{\mathcal{K}^\circ}^{1/(q-1)}.
\end{eqnarray*}
For $i=1,2$, $d_i\in\mathcal{K}^\circ$ so that $\|d_i\|_{\mathcal{K}^\circ}\leq 1$.
We then obtain
\begin{equation*}
I \leq  C\big\| d_1 - d_2\|d_1\|_{\mathcal{K}^\circ}/\|d_2\|_{\mathcal{K}^\circ}\big\|_{\mathcal{K}^\circ}^{1/(q-1)} + 2\|d_1 - d_2\|_{\mathcal{K}^\circ}^{1/(q-1)}.
\end{equation*}
Also, with the triangle inequality 
\[
\big\|d_1 - d_2\frac{\|d_1\|_{\mathcal{K}^\circ}}{\|d_2\|_{\mathcal{K}^\circ}}\big\|\leq \|d_1 - d_2\|_{\mathcal{K}^\circ} + \big\|d_2 - d_2\frac{\|d_1\|_{\mathcal{K}^\circ}}{\|d_2\|_{\mathcal{K}^\circ}}\big\| \leq \|d_1 - d_2\|_{\mathcal{K}^\circ} + \|d_2\|_{\mathcal{K}^\circ} - \|d_1\|_{\mathcal{K}^\circ} \leq 2 \|d_1 - d_2\|_{\mathcal{K}^\circ}.
\]
Hence, we finally obtain
\begin{equation}
\big\|\nabla \frac{1}{2}\|\cdot\|^2_{\mathcal{K}^\circ}(d_1) - \frac{1}{2}\nabla\|\cdot\|^2_{\mathcal{K}^\circ}(d_2)\big\|_{\mathcal{K}} \leq 2(C + 1)\big\|d_1 - d_2\big\|_{\mathcal{K}^\circ}^{1/(q-1)}.
\end{equation}
This equivalently means that $\frac{1}{2}\|\cdot\|_{\mathcal{K}}^{\circ}$ is $\big(2(C+1), 1 + 1/(q-1)\big)$-H\"older smooth as defined in \eqref{eq:L_Holder_smoothness_f}. 
Hence, since $q-1 = 1/(p-1)$, we get that $\frac{1}{2}\|\cdot\|_{\mathcal{K}}^{\circ}$ is $\big(2p(C+1), p\big)$-H\"older smooth.
\end{proof}

\end{document}